\documentclass{article} 
\usepackage{macros}
\usepackage[numbers]{natbib}
\usepackage{microtype}
\usepackage{booktabs} 
\usepackage{amsmath}
\usepackage{amssymb}
\usepackage{mathtools}
\usepackage{amsthm}

\usepackage[capitalize,noabbrev]{cleveref}

\theoremstyle{plain}
\newtheorem{theorem}{Theorem}[section]

\newtheorem{lemma}[theorem]{Lemma}

\theoremstyle{definition}
\newtheorem{definition}[theorem]{Definition}

\theoremstyle{remark}
\newtheorem{remark}[theorem]{Remark}

\title{Risk Estimation in a Markov Cost Process: \\Lower and Upper Bounds}

\author{
	Gugan Thoppe\\
	{\normalsize Indian Institute of Science} \\
	{\normalsize \texttt{gthoppe@iisc.ac.in}}
	\and 
	Prashanth L. A.\\
	{\normalsize Indian Institute of Technology Madras}\\
	{\normalsize \texttt{prashla@cse.iitm.ac.in}}
	\and
	Sanjay Bhat\\
	{\normalsize TCS Research, Hyderabad} \\
	{\normalsize \texttt{sanjay.bhat@tcs.com}}
}
\date{}
\begin{document}
	
	\maketitle

\begin{abstract}
We tackle the problem of estimating risk measures of the infinite-horizon discounted cost within a Markov cost process. The risk measures we study include variance, Value-at-Risk (\vart), and Conditional Value-at-Risk (\cvart).  First, we show that estimating any of these risk measures with $\epsilon$-accuracy, either in expected or high-probability sense, requires at least $\Omega(1/\epsilon^2)$ samples. Then, using a truncation scheme, we derive an upper bound for the CVaR and variance estimation. This bound matches our lower bound up to logarithmic factors. Finally, we discuss an extension of our estimation scheme that covers more general risk measures satisfying a certain continuity criterion, e.g., spectral risk measures, utility-based shortfall risk.  To the best of our knowledge, our work is the first to provide lower and upper bounds for estimating any risk measure beyond the mean within a Markovian setting. Our lower bounds also extend to the infinite-horizon discounted costs' mean. Even in that case, our lower bound  of $\Omega(1/\epsilon^2) $ improves upon the existing $\Omega(1/\epsilon)$ bound \cite{metelli2023tale}. 
\end{abstract}

\section{Introduction}
In a traditional discounted reinforcement learning (RL) problem \citep{BertsekasT96,sutton2018reinforcement}, the objective is to maximize the value function, which is the expected value of the infinite-horizon cumulative discounted reward. However, optimizing only the expected value is not appealing in several practical applications. For instance, in the financial domain, strategists like to consider the risk of investments. Similarly, in transportation, road users are sensitive to the variations in the delay incurred and would especially like to avoid a large delay even if it occurs infrequently. Risk-sensitive RL addresses such applications by incorporating a risk measure in the optimization process, either in the objective or as a constraint. 

Going beyond the expected value, three well-known risk measures are variance, Value-at-Risk (VaR) and Conditional Value-at-Risk (CVaR) \cite{rockafellar2000optimization}. For a Cumulative Distribution Function (CDF) $\cF,$ the VaR $v_\alpha(\cF)$ and CVaR $c_\alpha(\cF)$ at a given level $\alpha \in (0,1)$ is defined by
\begin{align}
	v_\alpha(\cF) &  = \inf \lbrace \xi : \prob{X \leq \xi} \geq \alpha \rbrace, \textrm{ and ~} \label{e:VaR.Defn} \\
	c_\alpha(\cF) & =  v_{\alpha}(X)  + \frac{\mathbb{E} \left[ X - v_{\alpha}(X) \right] ^+}{1 - \alpha} \label{eq:cvar-def},
\end{align}
where $X \sim \cF.$ From the above, it is apparent that $v_\alpha(\cF)$ is a certain quantile of the CDF $\cF.$ For a continuous distribution $\cF,$ $v_\alpha(\cF) = \cF^{-1}(\alpha)$,  while CVaR can be equivalently written as $c_\alpha(\cF) = \E[ X \mid X \ge v_\alpha(X)]$ for $X \sim \cF.$ In words, CVaR is the expected value of $X$, conditioned on the event that  $X$ exceeds the VaR. Choosing a $\alpha$ close to $1$ and taking $X$ as modeling the losses of a financial position, CVaR can be understood as the expected loss given that losses have exceeded a certain threshold (specified by a quantile).  In the financial domain, CVaR is preferred over VaR because CVaR is a coherent risk measure \cite{artzner1999coherent}, while VaR is not. In particular, with VaR as the risk measure, diversification can cause an increase in the risk as indicated by VaR -- an attribute that is not desirable in a financial risk measure.

With independent and identically distributed (i.i.d.) samples, the estimation of VaR and CVaR has received a lot of attention recently in the literature, cf. \cite{bhat2019concentration,brown2007large,wang2010deviation,thomas2019concentration,Kolla,prashanth2020concentration}. The concentration bounds for CVaR have been useful in deriving sample complexity results in the context of empirical risk minimization and bandit applications. 

In this paper, we are concerned with the problem of estimating a risk measure from a sample path of a discounted Markov Cost Process (MCP). In the context of RL, this is equivalent to the policy evaluation problem, albeit for a risk measure. For this problem, we derive minimax sample complexity lower bounds as well as an upper bound.  As a parallel to the i.i.d. sampling framework mentioned above, the upper bound we derive is useful in the design and analysis of risk-sensitive policy gradient algorithms (see \cite{prashanth2022risk} for a recent survey). 

\begin{table*}
	\centering
	\caption{Summary of the sample complexity lower and upper bounds, in an expected sense,  for estimating various risk measures.  
    For a given $\epsilon>0$,  sample complexity is the number of sample transitions $N$ such that the estimation error $\E|\hat \eta_n - \eta(D)|<\epsilon$ for all $n \geq N.$ Here $\eta$ is the risk measure, $D$ is the cumulative discounted cost, and $\hat \eta_N$ the risk estimate. Here $\tilde O\left(\cdot \right)$ is a variant of the big-O notation that ignores logarithmic factors.\\}
 \label{tab:summary}
 		\begin{tabular}{c|c|c|c}
			\hline
			\multirow{2}{*} {\textbf{Bound type}} &\multirow{2}{*}{\textbf{Risk measure}} & \multirow{2}{*}{\textbf{Sample complexity}} &  \multirow{2}{*}{\textbf{Reference}} \\ 
			&  & &  \\ \hline
    %
    \multirow{2}{*} {Lower bound} & \multirow{2}{*} {Mean, VaR, CVaR, variance} & \multirow{2}{*} {$\Omega \left(\frac{1}{\epsilon^2} \right)$} &  \multirow{2}{*} {Theorems~\ref{thm:lower.bound} and \ref{thm:Lower.Bd}}  \\
    &&&\\[-0.5ex]\hline
    
			\multirow{2}{*} {Upper bound} & \multirow{2}{*} {CVaR} & \multirow{2}{*} {$\widetilde{\mathcal{O}} \left(\frac{1}{\epsilon^2} \right)$} &  \multirow{2}{*} {Theorem \ref{thm:ub-cvar-expec}} \\ 
      &&& \\[-0.5ex]\hline
			\multirow{2}{*} {Upper bound} & \multirow{2}{*} {Lipschitz risk measure}  & \multirow{2}{*} {$\widetilde{\mathcal{O}} \left(\frac{1}{\epsilon^2} \right)$} &  \multirow{2}{*} {Theorem \ref{thm:expec-bd-generalrisk}}  \\
   &&&\\[-0.5ex]\hline
   			\multirow{2}{*} {Upper bound} & \multirow{2}{*} {Variance} & \multirow{2}{*} {$\widetilde{\mathcal{O}} \left(\frac{1}{\epsilon^2} \right)$} &  \multirow{2}{*} {Theorem \ref{thm:ub-variance-conc}} \\ 
      &&&\\\hline
		\end{tabular}
\end{table*}

Our key contributions are summarized below, while Tables \ref{tab:summary} and \ref{tab:summary-prob} provide a summary of the bounds we have derived.
\begin{enumerate}
    \item \textbf{Lower bounds}: We derive a minimax sample complexity lower bound of $\Omega(1/\epsilon^2)$ for risk estimation in two types of MCP problem instances: one with deterministic costs and the other with stochastic costs. In either case, our bound is order optimal and the first of its kind for risk estimation. Our first bound applies to \vart~and \cvart~of the infinite-horizon discounted cost, while the second applies \emph{even} to its mean and variance.
    


    \item \textbf{Lower bound proof}: We obtain our two lower bounds via novel proof techniques. In the deterministic costs case, the bound is obtained by identifying a `hard' problem instance involving a two-state Markov chain with a cost function that suitably diverges as $\epsilon \to 0.$ In particular, our bound is obtained by solving a suitable constrained optimization problem; in \cite{metelli2023tale}, the analogous optimization problem for mean estimation is unconstrained. With stochastic costs, we show that our lower bound holds even when the cost mean is bounded, as a function of $\epsilon.$ The `hard' problem instance in this case involves a MCP with a single state and a Gaussian single-stage cost function.  
      
    \item \textbf{Upper bound}: Using a truncated horizon estimation scheme, we derive an upper bound of $\widetilde{\mathcal{O}} \left(\frac{1}{\epsilon^2} \right)$ for CVaR  and variance estimation. These bounds match our corresponding lower bounds up to logarithmic factors. 

    \item \textbf{Other risk measures}: Finally, we also propose an extension of the estimation scheme to cover risk measures that satisfy a certain Lipschitz continuity criterion. Prominent risk measures that are covered in this extension are spectral risk measures \cite{acerbi2002spectral} and utility-based shortfall risk \cite{follmer2002convex}.
\end{enumerate}

We now compare our contributions to \cite{metelli2023tale}, which is the closest related work. In the aforementioned reference, the authors derive a minimax sample-complexity lower bound of $\Omega\left(\frac{1}{\epsilon}\right)$ in a probabilistic sense for estimating the mean of the infinite-horizon discounted cost of an MCP; see row 1 of Table~\ref{tab:summary-prob}. Their proof involves a two-state Markov chain with $\{0,1\}$ rewards. In this setting, the mean of the cumulative discounted cost can be explicitly written as a function of the transition probabilities.
In contrast, the proofs of our lower bounds are more challenging owing to the lack of a closed form expression for the risk measures we consider.  Moreover, our lower bounds, when specialized to mean estimation, leads to an improvement in comparison to \cite{metelli2023tale}.

The rest of the paper is organized as follows:
Section \ref{sec:pb} provides the problem formulation. Section \ref{sec:lb} presents two different lower bounds for estimation of risk measures that include VaR, CVaR, and variance. Section \ref{sec:ub} presents the upper bounds for estimation of risk measures that include CVaR, variance and a general class of Lipschitz risk measures. The lower and upper bounds include two variants, namely bounds which hold `in expectation' as well as bounds that hold `with high probability'.
Section \ref{sec:pf_sketch} provides the proofs.
Finally, Section \ref{sec:conclusions} provides the concluding remarks. 

\begin{table*}
	\centering
	\caption{Summary and comparison of sample complexity lower and upper bounds, in a high-probability sense, for various risk measures. For a given $\epsilon > 0$ and $\delta \in (0, 1),$ sample complexity is the number of state transitions $N$ such that $\Pr\{|\hat{\eta}_N - \eta(D)| > \epsilon\} \leq \delta.$ Here $\epsilon$ denotes the estimation accuracy, $\delta$ is the confidence, $\eta$ the risk measure, $D$ is the cumulative discounted cost, and  $\hat \eta_N$ the risk estimate using $N$ sample transitions of the MCP. In the bounds below, $T$ is the truncation parameter used in our risk estimator, $\gamma$ is the discount factor, and $K$ is an upper bound on the costs of the MCP. In the bounds, the constants $c, c'$ vary between rows. In the fourth row, $L$ denotes the Lipschitz constant for a risk measure (see \eqref{tonedef} below). Except the first row, other rows concern our work.\\}
 \label{tab:summary-prob}
		\begin{tabular}{c|c|c|c|c}
			\hline
			\textbf{Bound} &\textbf{Risk} & \multirow{2}{*}{\textbf{Tail bound}} & 
   \multirow{2}{*}{\textbf{Sample complexity}} & 
   \multirow{2}{*}{\textbf{Reference}} \\ 
	\textbf{type}		& \textbf{measure}  &   &\\ \hline
 			Lower  & \multirow{2}{*} {Mean} & \multirow{2}{*} {$\exp\left(-\frac{c N\epsilon}{1-\epsilon} \right)$} & \multirow{2}{*} {$\Omega \left(\epsilon^{-1} \ln\left(\frac{1}{\delta}\right) \right)$} & \multirow{2}{*} {\cite{metelli2023tale}}  \\
    bound &&&\\\hline
 			Lower  & Mean, VaR,  & \multirow{2}{*} {$\exp\left(-c\sqrt{N \epsilon^2} \right)$} & \multirow{2}{*} {$\Omega \left(\epsilon^{-2} \ln\left(\frac{1}{\delta}\right) \right)$} &\multirow{2}{*} {Theorems~\ref{thm:lower.bound}, \ref{thm:Lower.Bd}} \\ 
    bound &CVaR&&\\\hline
			Upper  & \multirow{2}{*}{VaR, CVaR} & \multirow{2}{*}{$\exp\left(- \frac{c N}{T} \left(\epsilon-\frac{\gamma^T K}{1-\gamma}\right)^2 \right)$} &  
   \multirow{2}{*} {$\widetilde{\mathcal{O}} \left(\epsilon^{-2}\ln (\frac{1}{\delta})\right)$}& \multirow{2}{*} {Theorems~\ref{thm:ub-cvar}, \ref{thm:ub-var}}\\
   bound&&&\\\hline
			Upper & Lipschitz    & \multirow{2}{*} {$\exp\left[- \frac{c N}{T} \left[\frac{1}{L}\left[\epsilon-\frac{\gamma^T K}{1-\gamma}\right]-\frac{c'\sqrt{T}}{\sqrt{N}}\right]^2 \right]$} &  
   \multirow{2}{*} {$\widetilde{\mathcal{O}} \left(\epsilon^{-2}\ln (\frac{1}{\delta})\right)$}& \multirow{2}{*} {Theorem \ref{thm:conc-bd-generalrisk}}  \\ 
   bound &risk measure&&\\\hline
			Upper  & \multirow{2}{*} {Variance} & \multirow{2}{*}{$\exp\left(- \frac{c N}{T} \left(\epsilon-\frac{2\gamma^T K^2}{(1-\gamma)^2}\right)^2 \right)$} & \multirow{2}{*} {$\widetilde{\mathcal{O}} \left(\epsilon^{-2}\ln (\frac{1}{\delta})\right)$}& \multirow{2}{*} {Theorem \ref{thm:ub-variance-conc}} \\ 
   bound&&&\\\hline
		\end{tabular}
\end{table*}

\section{Problem Formulation and Preliminaries}
\label{sec:pb}
We formally provide all our notations, describe our setup, and state our research goals here. We also give a detailed description of a general risk estimation algorithm.

\textbf{Notations}: $\cU$ denotes an arbitrary (possibly infinite) set, and $\fF$ a $\sigma$-algebra on $\cU$ containing all its singleton subsets. For $\cS \subseteq \cU,$ $\fF_\cS := \{A \cap \cS: A \in \fF\}$ is the induced $\sigma$-algebra on $\cS.$ Also, $\sP(\cS),$ $\sP(\bR),$ and $\sP(\{0, 1\})$ are the sets of probability measures on $(\cS, \fF_\cS),$ $(\bR, \cB(\bR)),$ and $(\{0, 1\}, 2^{\{0, 1\}}),$ respectively, where $\cB(\bR)$ is the Borel-$\sigma$-algebra on $\bR,$ while $2^{\{0, 1\}}$ is $\{0, 1\}$'s power set. Similarly, $\sB(\cS)$ denotes the set of $\cS \mapsto \sP(\bR)$ functions. 

The tuple $(M, f)$ denotes a Markov Cost Process (MCP). Here, $M \equiv (\cS, P, \nu)$ is a Markov chain on the state-space $\cS \subseteq \cU$ with transition kernel $P: \cS \to \sP(\cS)$ and initial state distribution $\nu \in \cP(\cS),$ while $f \in \sB(\cS)$ is a (possibly stochastic) cost function. Further, $P(\cdot|s)$ is the distribution mapped from $s$ under $P,$  and $(s_n)_{n \in \bZ_+}$ is a trajectory of $M.$ Also, $\gamma \in [0, 1)$ denotes a discount factor, and $\cF_t(M, f)$ the CDF of the cumulative discounted cost over the horizon $t,$ i.e.,
\begin{equation}
    \sum_{n = 0}^{t - 1} \gamma^n f_n, \text{ where } f_n \sim f(s_n).
    \label{eq:n-Horizon-disc-cost-rv}
\end{equation}

Finally,
\[
    \sM := \left\{(M, f): \cF_t(M, f) \text{ converges weakly to a CDF}\right\},
\]
and, for $(M, f) \in \sM,$  
\begin{equation}
    \label{eq:disc-cost-rv}
    \cF(M, f) := \lim_{t \to \infty} \cF_t(M, f).
\end{equation}
Also, for $(M, f) \in \sM,$ the expressions $\mu(M, f),$ $\V(M, f),$ $\var(M, f),$ and $\cvar(M, f),$ with $\alpha \in (0, 1),$ denote the mean, variance, \vart\ and \cvart\ (both at the confidence level $\alpha\in(0,1)$) of the CDF $\cF(M, f),$ respectively. 

\textbf{Setup}: We presume we have access to an MCP $(M, f) \in \sM,$ where $M$ and $f$ are unknown, but whose state and single-stage cost trajectory can be observed. We further presume that this MCP can be restarted from a state sampled from the initial state distribution as many times as we want.

%

We now describe a general risk estimation algorithm: one that can be used to estimate a desired risk measure $\eta(M, f).$

\begin{definition}[Risk Estimation Algorithm] A risk estimation algorithm $\sA$ is a tuple $(\rp, \h{\eta})$ made up of a reset policy $\rp$ and an estimator $\h{\eta}.$    
\end{definition} 

Loosely, $\rp$ specifies the rule to reset the given Markov chain $M$, i.e., stop its natural evolution under $P,$ and restart it from a state sampled from $\nu.$ On the other hand, $\h{\eta}$ describes how to transform the observations of states, resets, and single-stage costs into an estimate of the given risk measure. 

We now formally define the above two terms as in \cite{metelli2023tale}. For $n \in \bZ_+$ (the set of non-negative integers), let $\cH_n := (\cU \times \{0, 1\})^{n}$ be the set of all possible histories of states and reset decisions until time $n.$

\begin{definition}[Reset Policy]
    A reset policy $\rp \equiv (\rp_n)_{n \in \bZ_+}$ on $\cS$ is a sequence of functions where $\rp_n: \cH_{n} \times \cU  \mapsto \sP(\{0, 1\})$ maps $H \in \cH_{n}$ and $s \in \cU$ to a distribution $\rp_n(\cdot|H, s) \in \sP(\{0, 1\}).$ 
\end{definition}

\begin{definition}[Estimator]
    An estimator $\h{\eta} \equiv (\h{\eta}_n)_{n \in \bZ_+}$ is a sequence of functions where $\h{\eta}_n: \cH_n \times \sB(\cS) \mapsto \bR$ maps $H \in \cH_n$ and $f \in \sB(\cS)$ to a real number representing the estimate of $\eta(M, f).$ 
\end{definition}

Next, we describe the evolution dynamics of a Markov chain $M$ under the reset policy $\rp.$

\begin{definition}[Resetted Chain]
    Let $M \equiv (\cS, P, \nu)$ be a Markov chain, and $\rp$ a reset policy. Then, the corresponding resetted chain $\Mr$ is a stochastic process $(s_n, Y_n)_{n \in \bZ_+}$ taking values in $\cS \times \{0, 1\}$ such that
    
    \vspace{-1.5ex}

    \begin{enumerate}
        \item the initial state $s_0$ is sampled from $\nu;$

        \item for all $n \in \bZ_+,$  the reset decision $Y_n \in \{0, 1\}$ is drawn from  $r_n(\cdot|H_{n}, s_n)$ and the subsequent state $s_{n + 1}$ from $Y_n \nu + (1 - Y_n) P(\cdot|s_n),$  where $H_{n} := (s_0, Y_0, \ldots, s_{n - 1}, Y_{n - 1}) \in \cH_{n};$
    \end{enumerate}
    and these samples are drawn with independent randomness. The joint distribution of $H_n$ is denoted by $P^n_{M, \rp}.$    
\end{definition}

\begin{remark}
The sequence $(s_n)_{n \in \bZ_+}$ in $\Mr$ satisfies 
    \begin{align}
    \label{e:Mr.dynamics}
        \Pr(s_{n + 1} \in \cB|H_{n}, s_n)  = \rp_n(\{1\}|H_{n}, s_n) \nu(\cB) + \rp_n(\{ 0\}|H_n, s_n) P(\cB| s')
    \end{align}
    for any $\cB \in \fF_\cS.$
    As discussed in \cite{metelli2023tale}, this process is non-Markovian and non-stationary when the reset distribution $\rp_n$ depends on the history  $H_{n}.$
\end{remark}

\textbf{Research Goals}: Obtain lower and upper bounds on the samples needed to obtain an $\epsilon$-accurate estimate of a risk measure $\eta(M, f)$ related to the unknown underlying MCP $(M, f).$ Formally, our goal is to first obtain bounds on 
\begin{equation}
\label{e:exp.risk.defn}
    \inf_{\sA \equiv (\h{\eta}, \rp)} \hspace{0.5ex} \sup_{(M, f) \in \sM} \bE|\h{\eta}_n(H_n, f) - \eta(M, f)|
\end{equation}
and
\begin{equation}
\label{e:prob.risk.defn}
   \inf_{\sA} \sup_{(M, f)} \Pr\left\{|\h{\eta}_n(H_n, f) - \eta(M, f)| \geq \epsilon \right\}
\end{equation}
as a function of $n$ and $\epsilon,$ where $\bE$ and $\bP$ are with respect to $H_n \sim P^n_{M, \rp},$ and $\eta(M, f)$ is either $\mu(M, f), \V(M, f),$ or one of $\var(M, f)$ and $\cvar(M, f)$ for a given $\alpha \in (0, 1).$ We then aim to use these results to compute the desired sample complexity bounds.

\section{Lower Bounds: Risk Estimation Error}
\label{sec:lb}
We obtain risk estimation lower bounds for two different MCP problem instances: first, with deterministic costs and, second, with stochastic costs. These two cases are discussed in Subsections~\ref{s:det.rewards} and \ref{s:stoc.rewards}, respectively. 

\subsection{Deterministic Costs}
\label{s:det.rewards}
Throughout this subsection, we presume $f: \cS \to \bR,$ i.e., $f$ assigns a fixed real-valued cost to each state. 

\begin{theorem}[Minimax Lower Bound]
	\label{thm:lower.bound}
	For every MCP $(M, f) \in \sM,$ let the risk measure $\eta(M, f)$ be either $\cF(M, f)$'s \vart~ $\var(M, f)$ or \cvart~$\cvar(M, f)$ at a given $\alpha \in (0, 1).$  Then, for every $n \in \bN,$ error threshold $\epsilon > 0,$ and discount factor $\gamma \in [0, 1),$ 
	\begin{align}
		\label{e:prob.lower.bound}
		\inf_{\sA} \sup_{(M, f)} \bP\{|\h{\eta}(H_n, f) - \eta(M, f)| \geq \epsilon\} \geq \exp\left[-n\epsilon^2 \ln \left(\frac{1}{\alpha}\right) \ln \left(\frac{1}{\gamma}\right) 
		\right],
	\end{align}
	and
	\begin{align}
		\label{e:exp.lower.bound}
		\inf_{\sA} \sup_{(M, f)} \bE |\h{\eta}(H_n, f) - \eta(M, f)| \geq  \frac{1}{\sqrt{n}} \exp\left[-\ln \alpha \ln \gamma\right], 
	\end{align}
	where $\h{\eta}(H_n, f) \equiv \h{\eta}_n(H_n, f).$
\end{theorem}
\begin{proof}
	See Section \ref{sec:appendix-lb1-proof}.
\end{proof}
\begin{remark}
	\label{r:det.sample.complexity}
	By substituting $\epsilon = -\ln \delta/\sqrt{n \ln \alpha  \ln \gamma}$ in \eqref{e:prob.lower.bound} for a $\delta > 0,$ it follows that
	\[
	\inf_{\sA} \hspace{-0.25ex} \sup_{(M, f)}  \hspace{-0.25ex} \Pr\left\{|\h{\eta}(H_n, f) - \eta(M, f)| \geq \frac{-\ln \delta}{\sqrt{n \ln \alpha  \ln \gamma}} \right\}  \hspace{-0.25ex} \geq \delta.  \hspace{-1ex} 
	\]
	This implies that any \vart~or \cvart~estimation algorithm needs at least $\Omega(\epsilon^{-2})$ many samples to guarantee an $\epsilon$-accurate solution with probability $1 - \delta.$ Similarly, \eqref{e:exp.lower.bound} says that any algorithm would again need $\Omega(\epsilon^{-2})$ many samples to guarantee an $\epsilon$-accurate solution in an expected sense.
\end{remark}

\begin{remark}
	In Theorem~4.1 of \cite{metelli2023tale}, the authors claim a $O(1/\sqrt{n})$ minimax lower bound for estimating the mean of the infinite horizon discounted cost. However, this bound is misleading because the lower bound is derived assuming a certain quantity $\sigma^2_f$ is a constant. On closer inspection of their proof, it is apparent that $\sigma^2_f = \epsilon (1 - \epsilon),$ where $\epsilon$ is the estimation accuracy; see (34) in the Appendix there. When this $\epsilon$ dependence is factored in, it can be seen that the lower bound is only $O(1/n).$
\end{remark}

\subsection{Stochastic Costs}
\label{s:stoc.rewards}
In Section~\ref{s:det.rewards}, we looked at the case where the cost function $f$ was deterministic. However, to get the $\Omega(\epsilon^{-2})$ sample complexity, we require that the cost $f_1(A)$  increase suitably as $\epsilon$ decays to $0;$ see  \eqref{e:f1(A).choice}. In this subsection, we show that, by allowing the single-stage costs to be stochastic, we can obtain similar sample complexity lower bounds as in Theorem~\ref{thm:lower.bound} even when these costs have bounded mean. To derive these bounds, we use a radically novel proof idea to that of Theorem~\ref{thm:lower.bound} and also to the one used in \cite{metelli2023tale}. 

\begin{theorem}[\textbf{Lower Bound}]
\label{thm:Lower.Bd}
For every MCP $(M, f) \in \sM,$ let the risk measure $\eta(M, f)$ be either $\cF(M, f)$'s \vart~$\var(M, f)$ or \cvart\ $\cvar(M, f)$ at a given $\alpha \in (0, 1),$ or $\cF(M, f)$'s mean $\mu(M, f)$ or its variance $\V(M, f).$ Then, for every $n \in \bN$ and $\epsilon \in (0, 1/(2\sqrt{Kn}],$
\begin{align}
\label{e:high.prob.risk.Bd}
    \inf_{\sA} \sup_{(M, f)} \Pr\left\{|\h{\eta}(H_n) - \eta(M, f)| \geq \epsilon \right\} \geq \frac{1}{2}\exp\left(-2\sqrt{K n \epsilon^2} \right),
\end{align}
and
\begin{equation}
\label{e:exp.Bd}
    \inf_{\sA} \sup_{(M, f)} \bE|\h{\eta}(H_n) - \eta(M, f)| \geq \frac{1}{8 \sqrt{K n}},
\end{equation}
where $\h{\eta}(H_n) \equiv \h{\eta}_n(H_n)$ and $K = 6$ if $\eta(M, f) = \V(M, f)$ and $2$ otherwise.
\end{theorem}
\begin{proof}
See Section \ref{sec:appendix-lb2-proof}.
\end{proof}
\begin{remark}
By substituting $\epsilon = \ln[1/(2\delta)]/\sqrt{8n}$ in \eqref{e:high.prob.risk.Bd}, we have
\begin{align}
    \inf_{\sA} \hspace{-0.25ex} \sup_{(M, f)}  \hspace{-0.25ex} \Pr\left\{|\h{\eta}(H_n) - \eta(M, f)| \geq \frac{\ln [1/(2 \delta)]}{\sqrt{8n}}\right\}  \hspace{-0.25ex} \geq \delta.  \hspace{-1ex} \label{eq:lb-hpb-delta}
\end{align}
Hence, as in Remark~\ref{r:det.sample.complexity}, it follows that the sample complexity in the high-probability sense is $\Omega(\epsilon^{-2}).$ Similarly, the sample complexity in the expected sense is also $\Omega(\epsilon^{-2}).$
\end{remark}

\section{Upper Bounds for Risk Estimation Error}
\label{sec:ub}
\subsection{Upper Bound: VaR and CVaR}
\label{sec:ub-cvar}
In this section, we first present an estimation scheme for the VaR and CVaR of the cumulative discounted cost in an infinite-horizon Markov chain. Subsequently, we provide an estimation scheme for more general risk measures that satisfy a certain continuity criterion. 

We first present the classic estimators for VaR and CVaR of a random variable $X$ using $m$ i.i.d. samples \cite{serfling2009approximation}.
Let $\{X_1,\ldots,X_m\}$ denote the set of samples. Define the empirical distribution function (EDF) ${F}_m(\cdot)$ as follows:
\[
{F}_m (x) = \frac{1}{n} \sum_{i=1}^m \indic{X_i \leq x}, \forall
 x \in \mathbb{R}.
\]
Using EDF $F_m$, VaR and CVaR estimators, $\hat{v}_{m, \alpha}$ and $\hat{c}_{m, \alpha}$, respectively, are formed as follows:
\begin{equation}
\begin{split}
\hat{v}_{m, \alpha} & = F_m^{-1}(\alpha) = X_{\left[  \lceil m\alpha \rceil \right]}, \textrm{ and ~}\\
\hat{c}_{m, \alpha}  & = \frac{1}{m(1-\alpha)} \sum_{i=1}^m X_i \indic{X_i \geq \hat{v}_{m, \alpha}},    
\end{split}
\label{eq:var-cvar-estimate}
\end{equation}
where $X_{[i]}$ denotes the $i$th order statistic, for $i=1,\ldots,m$.

In the infinite-horizon discounted Markov chain setting that we consider in this paper, we estimate the VaR and CVaR of the cumulative discounted cost $D(f)$ \eqref{eq:disc-cost-rv} using a truncated estimator\footnote{For ease of notation, we drop the dependence on the cost function $f$, and use $D$ to denote the cumulative discounted cost random variable}. 
More precisely, given a budget of $N$ transitions, we reset after every $T$ steps to obtain $m=\left\lceil \frac{N}{T}\right\rceil$ trajectories. 
Since each trajectory provides an independent truncated sample of the cumulative cost $D$, we use the $m$ samples obtained from the distribution of $D$ to form the VaR and CVaR estimates using \eqref{eq:var-cvar-estimate}. More precisely,
let $D_1,\ldots,D_m$ denote the cumulative discounted cost samples obtained from the $m$ trajectories over the finite horizon $T$. With these samples, we form estimates $\hat v_N$ and $\hat c_N$ of VaR and CVaR of cumulative discounted cost r.v $D$.

\begin{theorem}[\textbf{Bound in expectation for CVaR}]
\label{thm:ub-cvar-expec}
	Assume $|f(s)|\le K$ for all $s\in \cS$. Let $\hat{c}_{N}$ denote the CVaR estimator formed using \eqref{eq:var-cvar-estimate} with a sample path of $N$ transitions and  truncation parameter $T$. Let $m=\left\lceil \frac{N}{T}\right\rceil$.  Then, we have
	\begin{align}
		&\E\left| \hat c_N - \cvar(D)\right| \le 
  \frac{32 (1-\gamma^T)^2 K^2}{(1-\alpha)(1-\gamma)^2 \sqrt{m}} + \frac{\gamma^T K}{1-\gamma}.
  \label{eq:expec-bd-cvar}
	\end{align}	
\end{theorem}
\begin{proof}
See Section \ref{sec:proof-ub-cvar-expec}.
\end{proof}
\begin{remark}
    Comparing the bound in expectation with the lower bound in \eqref{e:exp.Bd}, it is apparent that the bounds match in terms of the dependence on the length $n$ of the sample path, modulo an additional factor of $\frac{\gamma^T K}{1-\gamma}$ in the upper bound. The latter can be made small by choosing a larger truncation parameter $T$. 
\end{remark}

\begin{remark}\label{rem:4.3}
    For a given $\epsilon$, choose $T$ such that $\frac{\gamma^T K}{1-\gamma} < \frac{\epsilon}{2}$. Such a $T$ is $O\left(\log(1/\epsilon)\right)$.
    Separately, using the relation $1 - \gamma^T \leq 1,$ it can be seen that the first term on RHS of \eqref{eq:expec-bd-cvar} would fall below $\frac{\epsilon}{2}$ when $m = O\left(\frac{1}{\epsilon^2}\right)$. Now, since $N \leq m T,$ we have that the CVaR estimation scheme has a sample complexity $\tilde O\left(\frac{1}{\epsilon^2}\right)$, matching the corresponding result in the lower bound up to logarithmic factors. 
\end{remark}

The result below establishes an exponential concentration result for the CVaR estimate $\hat c_N$.
\begin{theorem}[\textbf{Concentration bound for CVaR}]
\label{thm:ub-cvar}
Under the assumptions of Theorem \ref{thm:ub-cvar-expec}, 	
for every $\epsilon>\frac{\gamma^T K}{1-\gamma}$, we have
\begin{equation}
 \label{eq:cvar-conc-bd}
 \begin{aligned}
		&\prob{| \hat c_N - \cvar(D)| > \epsilon} \le 
		6 \exp\left(- \frac{m (1-\alpha)(1-\gamma)^2}{11 (1-\gamma^{T})^2 K^2} \left(\epsilon-\frac{\gamma^T K}{1-\gamma}\right)^2\right),
\end{aligned}	
\end{equation}
where $m=\left\lceil \frac{N}{T}\right\rceil$.
\end{theorem}
\begin{proof}
See Section \ref{sec:proof-sec:proof-ub-cvar}.
\end{proof}

\begin{remark}
    Suppose
    \[
        L := \frac{(1 - \alpha)(1 - \gamma)^2}{11 K^2}.
    \]
    Then, since $(1 - \gamma^T) \leq 1,$ the bound in \eqref{eq:cvar-conc-bd} leads to 
    \[
        \prob{| \hat c_N - \cvar(D)| > \epsilon} \le 6\exp\left(-mL\left(\epsilon - \frac{\gamma^TK}{1 - \gamma}\right)^2\right)
    \]
    Choosing $T,$ $m,$ and, hence, $N$ along the lines discussed in Remark~\ref{rem:4.3} shows that, for $N = \tilde{O}(\epsilon^{-2}),$ we have $\prob{| \hat c_N - \cvar(D)| > \epsilon} \le \delta.$
\end{remark}

\begin{remark}
    The tail bound in \eqref{eq:cvar-conc-bd} can be inverted to arrive at the following high-confidence form: given $\delta \in (0,1)$, with probability (w.p.) $1-\delta$, we have
   \begin{align}		
|\hat c_N - \cvar(D)|\le \frac{K}{1-\gamma} \left[ (1-\gamma^{T})\sqrt{\frac{ 11  \log(\frac{6}{\delta})}{m (1-\alpha)}}  + \gamma^T\right]. \label{eq:cvar-hpb-form}
\end{align} 
\end{remark}

\begin{remark}
    Comparing the high-confidence bound in \eqref{eq:cvar-hpb-form} with the lower bound in \eqref{eq:lb-hpb-delta}, it is apparent that the bounds match in terms of the dependence on the length $n$ of the sample path. However, in terms of dependence on $\delta$, it can be seen that there is a gap as the lower bound has a $\log \frac{1}{\delta}$ factor, while the upper bound has $\sqrt{\log \frac{1}{\delta}}$. We believe the lower bound can be improved to close this gap on $\delta$-dependence.
\end{remark}

The last result of this section is a concentration bound for the VaR estimate $\hat v_N$, which is obtained by using a truncation parameter $T$. 
In the i.i.d. case, establishing a tail bound for VaR is difficult for distributions that are flat around the VaR, and a minimum growth rate assumption is usually imposed on the underlying distribution to overcome this problem, cf. \cite{kolla2019concentration,prashanth2020concentration}. In the case of the truncated estimator $\hat v_N$, the concentration bound is arrived by first relating the VaR of the truncated random variable $D_T$ and discounted cumulative cost $D$, followed by an application of i.i.d. concentration bound for VaR of $D_T$. For this approach to work, we require a minimum growth assumption on the distribution of $D_T$, which is formalized below.

\textbf{(A1)} The r.v. $D_T$ is continuous with a density, say $f_T$ satisfying the following for some $\eta,\zeta>0$:
$f_T(x)>\ell$ for all $x\in \left[\var(D_T) - \frac{\zeta}{2}, \var(D_T)+\frac{\zeta}{2}\right]$.

\begin{theorem}[\textbf{Concentration bound for VaR}]
\label{thm:ub-var}
Under (A1) and the assumptions of Theorem \ref{thm:ub-cvar-expec}, 	
for every $\epsilon>\frac{\gamma^T K}{1-\gamma}$, we have
\begin{equation}
 \label{eq:var-conc-bd}
 \begin{aligned}
		&\prob{| \hat v_N - \var(D)| > \epsilon} \le
		2 \exp\left(- 2m \ell^2  \min\left\{\left(\epsilon-\frac{\gamma^T K}{1-\gamma}\right)^2,\zeta^2\right\}\right),
\end{aligned}	
\end{equation}
where $m=\left\lceil \frac{N}{T}\right\rceil$.
\end{theorem}
\begin{proof}
See Section \ref{sec:proof-sec:proof-ub-var}.
\end{proof}


\subsection{Upper Bound: Lipschitz Risk Measures}
\label{sec:ub_ext}
In this section, we extend the truncation-based estimation scheme presented earlier to cover risk measures that satisfy a continuity criterion, which is made precise below.
\begin{definition}
\label{def:lip-risk}
Let $(\cL,W_1)$ denote the metric space of distributions, with the 1-Wasserstein distance as the metric $W_{1}$. 
A risk measure $\eta(\cdot)$ is said to be Lipschitz-continuous if there exists $L>0$ such that, for any two distributions $F, G \in \cL$, the following holds:
\begin{align}
	\left| \eta(F) - \eta(G)\right| \le L W_1(F,G). \label{tonedef}
\end{align} 
\end{definition}
In \cite{la2022wasserstein}, the authors establish that optimized certainty equivalent (OCE) risk \cite{bental1986oce,bental2007oce} that includes CVaR, spectral risk measure \cite{acerbi2002spectral}, and utility-based shortfall risk \cite{follmer2002convex} belong to the class of Lipschitz risk measures, see Lemmas 12, 13 and 15 in \cite{la2022wasserstein}. 

As before, let $D$ denote the cumulative cost, and $D_1,\ldots, D_m$ denote the samples obtained from $m$ truncated trajectories. Let $\eta(D)$ be a Lipschitz risk measure. Using the samples with EDF $F_m$, we form the following estimate:
\begin{align}
    \hat\eta_N = \eta(F_m).\label{eq:rhon}
\end{align}
For the case of spectral risk measure and utility-based shortfall risk, the estimate defined above coincides with their classic estimators, cf. \cite{hu2018utility,la2022wasserstein}. Moreover, the CVaR estimator defined in \eqref{eq:var-cvar-estimate} is a special case of \eqref{eq:rhon}.

The result below provides a bound in expectation for the estimator \eqref{eq:rhon} of a Lipschitz risk measure. 
\begin{theorem}[\textbf{Bound in expectation}]
	\label{thm:expec-bd-generalrisk}
	Assume $|f(s)|\le K$ for all $s\in \cS$. 
	Let $\eta(\cdot)$ denote a Lipschitz risk measure, which satisfies the following properties: (i) $\eta(X+a)=\eta(X) + a$ for any $a\in \bR$; and (ii) $\eta(X) \le \eta(Y)$ if $X\le Y$. 
	Let $\eta_{N}$ denote the estimator \eqref{eq:rhon} formed using a sample path of $N$ transitions, with truncated horizon $T$.   
	Then,  we have
	\begin{align*}
		&\E\left| \hat\eta_N - \eta(D)\right|  \le \frac{32 L (1-\gamma^T)^2 K^2}{(1-\gamma)^2 \sqrt{m}} + \frac{\gamma^T K}{1-\gamma},
	\end{align*}	
	where $m=\left\lceil \frac{N}{T}\right\rceil$, and $\myexp$ is the Euler constant.
\end{theorem}
\begin{proof}
Follows in a similar manner as the proof of Theorem \ref{thm:ub-cvar-expec}. The reader is referred to Section \ref{sec:appendix-ub-lipschitz-proof} for the details.
\end{proof}

The result below provides a concentration bound for the estimator $\hat\eta_N$ defined in\eqref{eq:rhon}.
\begin{theorem}[\textbf{Concentration bound}]
	\label{thm:conc-bd-generalrisk}
	Under conditions of Theorem \ref{thm:expec-bd-generalrisk},
for every $\epsilon$ such that 
 $ \frac{512K}{(1 - \gamma) \sqrt{m}}<\frac{1}{L}\left(\epsilon-\frac{\gamma^T K}{1-\gamma}\right)< \frac{512K}{(1-\gamma)\sqrt{m}}+16K\sqrt{\myexp}$, we have
	\begin{align*}
		&\prob{| \hat\eta_N - \eta(D)| > \epsilon} \le 2 \exp\left[- \frac{2m(1-\gamma)^2}{256 (1-\gamma^{T})^2 K^2\myexp}  \left[\frac{1}{L} \left[\epsilon-\frac{\gamma^T K}{1-\gamma}\right]-\frac{512K}{(1-\gamma)\sqrt{m}}\right]^2\right],
	\end{align*}	
where $m=\left\lceil \frac{N}{T}\right\rceil$, and $\myexp$ is the Euler constant.
\end{theorem}
\begin{proof}
Follows in a similar manner as the proof of Theorem \ref{thm:ub-cvar}. The reader is referred to Section \ref{sec:appendix-ub-lipschitz-proof-conc} for the details.
\end{proof}

\subsection{Upper Bound: Variance}
Recall the truncation-based scheme from Section \ref{sec:ub-cvar}, where $m=\left\lceil \frac{N}{T}\right\rceil$ independent trajectories were obtained with the truncation parameter $T$. From these trajectories, we obtain samples $D_1,\ldots,D_m$ of the truncated discounted cost $D_T= \sum_{t = 0}^{T-1} \gamma^t f(s_t)$. 
Using these samples, we form the estimate $\hat \V_N$ of the variance $\V(D)$ of the cumulative discounted cost as follows:  
\begin{align}
    \hat \V_N = \frac{1}{m-1} \sum_{i=1}^{m} \left(D_i - \bar \kappa_m\right)^2, \textrm{ where }\bar \kappa_m = \frac1{m}\sum_{i=1}^m D_i.\label{eq:variance-est}
\end{align}
Notice that variance is not a Lipschitz risk measures in the sense of Definition \ref{def:lip-risk}. Thus, the result in Theorem \ref{thm:conc-bd-generalrisk} does not apply for variance. However, using the same proof idea, we can infer a concentration bound for the variance estimator in \eqref{eq:variance-est}. This result is presented below, while the proof is available in Appendix \ref{sec:appendix-ub-variance-proof-conc}. 
\begin{theorem}
\label{thm:ub-variance-conc}
	Assume $|f(s)|\le K$ for all $s\in \cS$. Let $\hat \V_{N}$ denote the variance estimator formed using  \eqref{eq:variance-est}.
 Let $m=\left\lceil \frac{N}{T}\right\rceil$. Then, 
\begin{align}
    \E|\hat \V_N - V(D)| 
    & \le \frac{2(1-\gamma^T)^2 K^2}{\sqrt{m}(1-\gamma)^2} + \frac{4\gamma^T K^2}{(1-\gamma)^2}.
\end{align}
 In addition,
for every $\epsilon>\frac{4\gamma^T K^2}{(1-\gamma)^2}$, we have
\begin{equation}
 \label{eq:variance-conc-bd}
 \begin{aligned}
		&\prob{| \hat \V_N - \V(D)| > \epsilon} \le 
		2 \exp\left(- \frac{ m (1-\gamma)^4}{ 2(1-\gamma^{T})^4 K^4} \left(\epsilon-\frac{4\gamma^T K^2}{(1-\gamma)^2}\right)^2\right).
\end{aligned}	
\end{equation} 
\end{theorem}
\begin{proof}
	See Section \ref{sec:appendix-ub-variance-proof-conc}.
\end{proof}


\section{Proofs}
\label{sec:pf_sketch}
\subsection{Proof of Theorem~\ref{thm:lower.bound}}
\label{sec:appendix-lb1-proof}
\begin{proof}
	We first derive the two lower bounds for $\eta(M, f) = \var(M, f).$ We later show how a similar proof tactic works in the \cvart~case as well. 
	
	Our proof of the \vart~result involves three main steps: i) constructing a suitable two-state MCP, ii) deriving a lower bound on the minimax error in terms of the MCP parameters, and iii) tightening the lower bound by optimizing over the various choices for these parameters. 
	While steps i) and ii) above are similar to those employed for deriving lower bounds for the expected value objective in \cite{metelli2023tale}, step iii) involves significant deviations owing to an inequality constraint on the VaR (see \eqref{e:v.alpha.constraint} below). Such a constraint is not present in the case of the expected value, making solution of the corresponding optimization problem simpler. 
	
	We now provide the details of these three steps. 
	
	\begin{enumerate}
		\item \textbf{MCP construction}: Let $A$ and $B$ be two arbitrary but distinct elements of $\cU.$ The  MCP we construct is   
		\begin{equation}
			\label{e:worst.case.MC}
			(\Ms, \fs),
		\end{equation}
		where $\Ms \equiv (\cS_*, P_*, \nu_*)$ is the two-state Markov chain
		with $\cS_* = \{A, B\},$ the transition matrix 
		\[
		P_* = \begin{pmatrix}
			p & 1 - p \\
			p & 1 - p
		\end{pmatrix}
		\]
		for some $p \in [0, 1],$ and the initial state distribution 
		$\nu_* \equiv (q, 1 - q)$ for some $q \in [0, 1];$ this Markov chain is shown in Figure~\ref{fig:mc}. We let the single-stage cost function $\fs$ be either $f_0$ or $f_1,$ where $f_0(A) = f_0(B) = f_1(B) =  0,$ while 
		\begin{equation}
			\label{e:f1(A).choice}
			f_1(A) = 2\epsilon \exp\left(\frac{1}{\epsilon^2} \right).
		\end{equation}
		Clearly, for $\alpha \in [0, 1],$ $\var(\Ms, f_0) = 0,$ while $\var(\Ms, f_1) \geq 0.$ 
		
		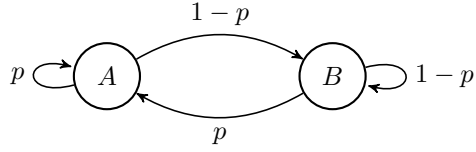
\begin{figure}[h]
			\centering
			\begin{tikzpicture}[->, >=stealth', auto, semithick, node distance=3cm]
				\tikzstyle{every state}=[fill=white,draw=black,thick,text=black,scale=1]
				\node[state]    (A)                     {$A$};
				\node[state]    (B)[right of=A]   {$B$};
				\path
				(A) edge[loop left]			node{$p $}	(A)
				edge[bend left,above]		node{$1 - p$}	(B)
				(B) edge[bend left,below]	node{$p$}	(A)
				edge[loop right]		node{$1-p$}	(B);
			\end{tikzpicture}
			\caption{A two state Markov chain}
			\label{fig:mc}
		\end{figure}
		
		\item \textbf{Lower bounding the minimax error}: Let $\sA \equiv (\eta, \rp)$ be any \vart~estimation algorithm, and $\Pr_*$ the probability distribution of $H_n$ under $P^n_{\Ms, \rp}.$ Then, for any $p$ and $q$ such that
		\begin{equation}
			\label{e:var.gap}
			\var(\Ms, f_1) = \var(\Ms, f_1) - \var(\Ms, f_0) \geq 2\epsilon,
		\end{equation} 
		we have
		\begin{align}
			\sup_{M,f}   \underset{H_n \sim P^n_{M, \rp}}{\Pr} &\left(|\h{\eta}(H_n, f)  - \var(M, f)| \ge \epsilon\right) \nonumber \\
			& \geq \max_{f \in \{f_0, f_1\}} \; \Pr_* \left(|\h{\eta}(H_n, f)  - \var(M, f)| \ge \epsilon\right) \\
			& \geq \frac{1}{2} \sum_{i = 0}^{1} \Pr_*\left(|\h{\eta}(H_n, f_i) - \var(\Ms, f_i)| \ge \epsilon\right) \label{e:max.bigger.than.average}\\
			& \ge \frac{1}{2}\Pr_*\big(\{|\h{\eta}(H_n, f_0) - \var(\Ms, f_0)| \geq \epsilon \}\nonumber \\
			& \hspace{4em} \cup \{|\h{\eta}(H_n, f_1) - \var(\Ms, f_1)| \geq \epsilon \} \big)\label{e:union.bound} \\
			&\ge \frac{1}{2}\Prs{\h{\eta}(H_n, f_0) =\h{\eta}(H_n, f_1)}\numberthis\label{eq:same.value.for.both.eestimators}\\
			& \ge \frac{1}{2}\Prs{s_t = B~\forall t \in \{0, 1, \ldots, n-1\} } \label{e:Special.Case.Undistinguishable.Estimators}\\
			& \ge \frac{1}{2} \nu_*(B) \min\{\nu_*(B), P_*(B|B)\}^{n - 1} \numberthis \label{e:L.Bd.reset.No.reset} \\
			& \ge \frac{1}{2} \min\{1 - q, 1 - p\}^n \label{e:resetted.chain.consequence.n}.
		\end{align}
		Above, \eqref{e:max.bigger.than.average} follows since $\max\{x, y\} \geq (x + y)/2$ for any $x, y \in \bR,$ while \eqref{e:union.bound} is due to an union bound. The relation in \eqref{eq:same.value.for.both.eestimators} holds because
		\[
		\{\h{\eta}(H_n, f_0) =\h{\eta}(H_n, f_1)\} \subseteq
		\left\{|\h{\eta}(H_n, f_0) - \var(\Ms, f_0)| \ge \epsilon\right\} 
		\cup 
		\left\{|\h{\eta}(H_n, f_1) - \var(\Ms, f_1)| \ge \epsilon\right\}
		\]
		which itself is implied by \eqref{e:var.gap}. Further, \eqref{e:Special.Case.Undistinguishable.Estimators} holds because the estimator $\h{\eta}$ cannot differentiate between $f_0$ and $f_1$ when $s_t = B$ $\forall t \in\{0, \ldots, n - 1\},$ while \eqref{e:L.Bd.reset.No.reset} holds due to \eqref{e:Mr.dynamics}. Finally, \eqref{e:resetted.chain.consequence.n} follows by using the fact that $\nu_*(B) = 1 - q$ and $P_*(B|B) = 1 - p.$

		\item \textbf{Tightening the bound}: Since \eqref{e:resetted.chain.consequence.n} is true for any $p$ and $q$ such that \eqref{e:var.gap} holds, we can now optimize over these values. This leads to following optimization problem: 
		\begin{equation}   
			\label{e:v.alpha.constraint}
			\begin{aligned}
				\max_{p, q} & \enspace \min\{1 - q, 1 - p\} \\
				\text{s.t. } &  \,0 \leq p \leq 1, \enspace
				0 \leq q \leq 1,  
				v_\alpha(\Ms, f_1) \geq 2 \epsilon 
			\end{aligned}
		\end{equation}
		Unlike the expected value case considered in \cite{metelli2023tale},  the optimal $p$ and $q$ cannot be inferred in closed form, and our proof involves a non-trivial argument using the VaR statistic to arrive values for $p$ and $q$ that suitably lower bound the max-min optimization problem  in \eqref{e:v.alpha.constraint} above.   We formalize this claim below. 
		
		\textbf{Claim}: Let $i := \left\lfloor \frac{1}{\epsilon^2 \ln \left(\frac{1}{\gamma}\right)} \right\rfloor.$  Then, $p = q = 1 - \alpha^{1/i}$ lies in the constraint set of \eqref{e:v.alpha.constraint}. 
		
		We now verify this claim. Clearly, the inequality $\var(\Ms, f_1) \geq 2 \epsilon$ is implied by
		\begin{equation}
			\label{e:VaR.alpha.Constraint}
			\alpha \geq \cF_1\left(2 \epsilon\right),
		\end{equation}
		where $\cF_1$ is a shorthand for the CDF $\cF(\Ms, f_1).$
		Therefore, to establish the claim, it suffices to show that \eqref{e:VaR.alpha.Constraint} holds for the given choice of $p$ and $q.$ 
		
		We have 
		\begin{align}
			\cF_1(2\epsilon) = {} & \lim_{t \to \infty} \Prob{\sum_{n = 0}^t \gamma^n f_1(s_n) \leq 2\epsilon} \nonumber \\
			\leq {} & \sup_{t \geq i - 1} \Prob{\sum_{n = 0}^t \gamma^n f_1(s_n) \leq 2\epsilon} \nonumber \\
			\leq {} & \Prob{s_n = B \enspace \forall n \in \{0, 1, \ldots, i - 1\}} \label{e:initial.i.states.B.implies.small.D(f)} \\
			\leq {} & \max\{1 - q, 1 - p\}^i \label{e:resetted.chain.consequence.i}
		\end{align}
		where \eqref{e:initial.i.states.B.implies.small.D(f)} holds because $s_n = A$ for even one $n \in \{0, \ldots, i - 1\}$ would imply that $\sum_{n = 0}^t \gamma^n f(s_n) > 2\epsilon$ for any $t \geq i - 1,$ which itself is implied by the fact that, $\forall n \in \{0, \ldots, i - 1\},$
		\[
		\gamma^n f_1(A) \geq \gamma^{i - 1} f_1(A) > \gamma^i f_1(A) \geq 2 \epsilon;
		\]
		while \eqref{e:resetted.chain.consequence.i} follows as in \eqref{e:resetted.chain.consequence.n}. Substituting the value of $p$ and $q$ from the claim in \eqref{e:resetted.chain.consequence.i} shows that $\cF_1(2\epsilon) \leq \alpha,$ as desired. 
		
		Using the claim, it now follows that the optimal value of \eqref{e:v.alpha.constraint} has the lower bound $\alpha^{1/i}.$ Combining this observation with \eqref{e:resetted.chain.consequence.n} and substituting the value of $i$ gives 
		\begin{equation}
			\sup_{M,f} \;  \underset{H_n \sim P^n_{M, \rp}}{\Pr}\left(|\h{\eta}(H_n, f) - \var(M, f)| \ge \epsilon\right)  \geq \alpha^{n/i} 
			\geq \exp\left[-n\epsilon^2 \ln \left(\frac{1}{\alpha}\right) \ln \left(\frac{1}{\gamma}\right) 
			\right].
		\end{equation}
	\end{enumerate}
	Since algorithm $\sA$ was arbitrary, \eqref{e:prob.lower.bound} follows. 
	
	We now derive \eqref{e:exp.lower.bound}. Clearly, for any random variable $X,$
	\[
	\bE|X| \geq \epsilon \Pr\{|X| \geq \epsilon\}.
	\]
	Hence, for any \vart~estimation algorithm $\sA,$
	\begin{align}
		\sup_{M, f} \underset{H_n \sim P^n_{M, \rp}}{\bE} &|\h{\eta}(H_n, f) - \var(M, f)| \nonumber \\ \geq {} & \epsilon \exp\left(-n\epsilon^2 \ln \left(\frac{1}{\alpha}\right) \ln \left(\frac{1}{\gamma}\right)\right) \nonumber \\
		\geq {} & \frac{e^{-\ln \alpha \ln \gamma}}{\sqrt{n}},
	\end{align}
	where the last relation follows by picking $\epsilon = n^{-1/2}.$ Again, since $\sA$ was arbitrary, \eqref{e:exp.lower.bound} follows. 
	
	It now remains to derive the lower bounds for the \cvart~case. The above proof works in more or less the same way, except for some minor modifications. In the \cvart~case, consider the optimization problem that is analogous to \eqref{e:v.alpha.constraint}. Due to \eqref{eq:cvar-def}, any $(p, q)$-pair that is feasible for \eqref{e:v.alpha.constraint} is also feasible for this new optimization problem. Hence, the \vart~lower bounds hold for \cvart~estimation as well. 
\end{proof}


\subsection{Proof of Theorem~\ref{thm:Lower.Bd}} 
\label{sec:appendix-lb2-proof}
\begin{proof}
We consider two variants of the problem instance given by the MCP $(M_0, f),$ where\\[0.5ex]
\begin{inparaenum}[\itshape 1)\upshape]
    \item the Markov chain is $M_0 \equiv (\cS, P, \nu)$ with $\cS$ being an arbitrary but fixed singleton subset of $\cU$ (say $\{s\}$), implying that $P$ and $\nu$ are trivial distributions; and 

    \item the single-stage cost function is
    \begin{equation}
    \label{e:f.distribution}
        f(s) \sim \cN((1 - \gamma)\mu, (1 - \gamma^2)\sigma^2),
    \end{equation}
    for some unknown $\mu \in \bR$ and $\sigma^2 > 0.$
\end{inparaenum}

In general, for any estimation algorithm $\sA \equiv (\h{\eta}, \rp),$ its estimate at time $n \in \bZ_+$ depends fully on the history $H_n \equiv (s_0, Y_0, \ldots, s_{n - 1}, Y_{n - 1})$ and the single-stage costs $f_0, \ldots, f_{n - 1},$ where $f_i \in f(s_i).$ However, in the case of $(M_0, f),$ the state-space $\cS$ underlying $M_0$ is trivial. Therefore, $s_0 = \cdots = s_{n - 1} = s.$ Similarly, whatever be the reset decision at any time $n,$ the algorithm gets to only see state $s$ and the associated random cost (which is independent of everything else) at $n$. Hence, the estimate $\h{\eta}(H_n, f)$ is essentially a function of $f_0, \ldots, f_{n - 1}$ i.e., 
\begin{equation}
\label{e:function.of.only.rewards}
    \h{\eta}(H_n, f) \equiv \h{\eta}(f_0, \ldots, f_{n - 1}).
\end{equation}
Separately, the structure of $(M_0, f)$ implies that $f_0, \ldots, f_{n - 1}$ are IID samples  with the distribution in \eqref{e:f.distribution}.

We break the rest of our proof into three parts. First, we show that the infinite-horizon cumulative discounted cost's CDF $\cF(M_0, f)$ (see \eqref{eq:disc-cost-rv}) exists for the above MCP and obtain expressions for its mean, variance, \vart, and \cvart. Next, we establish a relationship between $\cF(M_0, f)$ and the CDF of $f(s)$ given in \eqref{e:f.distribution}. Finally, we build upon the discussions in \cite{Duchi2023statistics} to obtain the stated lower bounds. 

Clearly, the $n$-horizon CDF $\cF_n(M_0, f),$ $n \in \bN,$ equals the CDF of $\sum_{t = 0}^{n - 1} \gamma^t X_t,$ where $(X_t)$ is an IID sequence of random variables having the same distribution as $f(s).$ However, the CDF of $\sum_{t = 0}^{n - 1} \gamma^t X_t$ is that of  $\cN((1 - \gamma^n) \mu, (1 - \gamma^{2n})\sigma^2).$  Further, its limit $\cF(M_0, f),$ as $n \to \infty,$ exists and is the CDF of $\cN(\mu, \sigma^2).$  Hence, 
\begin{equation}
\label{e:mean.VaR.CVaR.Formula}
    \begin{aligned}
        \mu(M_0, f) = {} & \mu, 
        \V(M_0, f) = {}  \sigma^2, \\
        \var(M_0, f) = {} & \mu + \sigma \Phi^{-1}(\alpha), \\
        \cvar(M_0, f) = {} & \mu + \sigma \frac{\phi (\Phi^{-1}(\alpha))}{1 - \alpha},
    \end{aligned}
\end{equation}
where $\phi$ and $\Phi$ are the PDF and CDF of the standard normal distribution, respectively. 

We next discuss a relationship between the samples of $\cF(M_0, f)$ and the distribution of $f(s).$ Let 
\begin{equation}
\label{e:a.b.defn}
\begin{aligned}
    a = {} & \frac{1 - \gamma + \sqrt{(1 - \gamma) (1 + 3 \gamma)}}{2}, 
    \text{ and } \,\,
    b = {}  \frac{1 - \gamma - \sqrt{(1 - \gamma) (1 + 3 \gamma)}}{2}.
\end{aligned}
\end{equation}
Then, for any two independent samples $Z$ and $Z'$ of $\cF(M_0, f),$ i.e., of $\cN(\mu, \sigma^2),$ the  sum 
\begin{equation}
\label{e:lc.combination.infinite.horizon.DF.samples}
    a Z + b Z' \sim \cN((1 - \gamma)\mu, (1 - \gamma^2) \sigma^2),
\end{equation}
i.e., it has the same distribution as $f(s);$ see \eqref{e:f.distribution}. 

Thus, any function $\h{\eta}$ that uses the $n$ single-stage costs $f_0, \ldots, f_{n - 1}$ to estimate $\eta(M, f)$ can be viewed as an hypothetical function $\h{\eta}'$ that works with $2 n$ IID samples $Z_0, \ldots, Z_{2n - 1}$ of the limiting distribution $\cF(M, f),$ i.e., of $\cN(\mu, \sigma^2),$ and adopts the following two steps: 
%
%

\begin{inparaenum}[\itshape 1)\upshape]
    \item combine every successive pair $Z_{2i}$ and $Z_{2i + 1},$ $i \in \{0, \ldots, n - 1\},$  to get $X_{i} = a Z_{2i} + b Z_{2i + 1},$ where $a$ and $b$ are as in \eqref{e:a.b.defn}; 

    \item use $\h{\eta}(X_0, \ldots, X_{n - 1})$ to get the final estimate.
\end{inparaenum}

In other words, 
\begin{equation}
\label{e:eta.etap.relation}
    \h{\eta}'(Z_0, \ldots, Z_{2n - 1}) = \h{\eta}(X_0, \ldots, X_{n - 1}).
\end{equation}

Finally, we derive our stated lower bound. The proofs for mean, \vart, and \cvart~are similar. Hence, we discuss them together now. Consider two different variants $(M_0, f^{+1})$ and $(M_0, f^{-1})$ of the MCP $(M_0, f),$ where $f^{\nu}(s),$ $\nu \in \{-1, +1\},$ is as in \eqref{e:f.distribution}, but with $\mu$ replaced by $\mu_{\nu} := \nu \epsilon$ for some $\epsilon > 0,$ and $\sigma^2 = 1.$ It follows from \eqref{e:mean.VaR.CVaR.Formula} that, for any of our quantities of  interest, i.e.,  mean, \vart, or \cvart,
\begin{equation}
    \eta(M_0, f^{+1}) - \eta(M_0, f^{-1}) = 2\epsilon.
\end{equation}

For $\nu \in \{-1, +1\},$ let $P_{\nu}$ denote the distribution $\cN(\mu_\nu, \sigma^2).$ Further, let $P^{k}_{\nu}$ denote the joint distribution of $Z_0, \ldots, Z_{k - 1},$ sampled independently from $\cN(\mu_{\nu}, \sigma^2).$ Also, let $\|\cdot\|_{\TV}$ denote the total variation distance and $D_{kl}$ the KL-divergence. Then, for every $n \in \bN$ and  $\epsilon \in (0, 1/\sqrt{8n}],$ we have
\begin{align}
     & \inf_{\sA} \sup_{M, f}
     \Pr\{|\h{\eta}(H_n, f)  - \eta(M, f)| \geq  \epsilon\} \label{e:high.prob.lower.Bd} \\
     \geq {} & \inf_{\sA} \hspace{-0.5ex} \sup_{\nu = -1, 1}  \Pr\{|\h{\eta}(H_n, f^\nu)  - \eta(M_0, f^\nu)| \geq  \epsilon\} \label{e:two.problem.instances} \\
     = {}  & \inf_{\h{\eta}} \hspace{-0.5ex} \sup_{\nu = -1, 1}
     \Pr\{|\h{\eta}(f^\nu_0, \ldots, f^{\nu}_{n - 1})  - \eta(M_0, f^\nu)| \hspace{-0.25ex} \geq \hspace{-0.5ex} \epsilon\} \hspace{-1ex} \label{e:alg.estimator.reln} \\
     \geq {} & \inf_{\h{\eta}'} \hspace{-1ex} \sup_{\nu = -1, 1} \hspace{-1ex} \Pr\{|\h{\eta}'(Z_0, \ldots, Z_{2n - 1})  \hspace{-0.25ex} - \hspace{-0.25ex}\eta(M_0, f^\nu)| \hspace{-0.25ex} \geq \hspace{-0.5ex} \epsilon\} \hspace{-1ex} \label{e:eta.eta.hyp.relation}\\
     = {} & \frac{1}{2} \left[1 -  \|P_{+1}^{2n} - P_{-1}^{2n}\|_{\TV} \right] \label{e:Lower.bound.TV.Variance}  \\
    \geq {} & \frac{1}{2} \left[1 -\sqrt{ \frac{1}{2} D_{\KL}(P^{2n}_{+1}, P^{2n}_{-1})}\right] \label{e:Pinskers.inequality}\\
    \geq {} & \frac{1}{2} \left[1 - \sqrt{n D_{\KL}(P_{+1}, P_{-1})}\right] \label{e:KL.Divergence.ind.samples} \\
    = {} & \frac{1}{2}[1 - \sqrt{2n \epsilon^2}] \label{e:KL.Divergence.Normal}\\
    \geq {} & \frac{1}{2} \exp\left(-\frac{\sqrt{2n \epsilon^2}}{1 - \sqrt{2n\epsilon^2}}\right) \label{e:1-x.lower.bound}\\
    \geq {} & \frac{1}{2} \exp\left(-2 \sqrt{2n \epsilon^2}\right), \label{e:eps.constraint}
\end{align}
which gives the relation in \eqref{e:high.prob.risk.Bd}, as desired. Above, \eqref{e:two.problem.instances} holds since we only consider two specific problem instances, \eqref{e:alg.estimator.reln} holds due to \eqref{e:function.of.only.rewards},  \eqref{e:eta.eta.hyp.relation} follows from \eqref{e:eta.etap.relation}, \eqref{e:Lower.bound.TV.Variance} follows from Eqs. (8.2.1) and (8.3.1) of \cite{Duchi2023statistics}, \eqref{e:Pinskers.inequality} follows from Pinsker's inequality, \eqref{e:KL.Divergence.ind.samples} follows from the independence of $Z_0, \ldots, Z_{2N - 1}$, \eqref{e:KL.Divergence.Normal} follows using the KL divergence formula for two Gaussian distributions, \eqref{e:1-x.lower.bound} follows since $2n\epsilon^2 \leq 1$ and $1 - x \geq e^{-x/(1 -x)}$ for every $x \leq 1$, while \eqref{e:eps.constraint} holds due to the constraint on $\epsilon.$ Also, note that the $\bP$ in \eqref{e:high.prob.lower.Bd} is with respect to $H_n \sim P^n_{M, \rp},$ while it is with respect to $P^n_{M_0, \rp}$ in 
\eqref{e:two.problem.instances}. Similarly, the $\bP$ in \eqref{e:eta.eta.hyp.relation} is with respect to $Z_0, \ldots, Z_{2n - 1}$ sampled in an IID fashion from $\cN(\mu_\nu, 1).$

We now prove~\eqref{e:exp.Bd}. Combining the discussions from Theorem~\ref{thm:lower.bound} and \eqref{e:KL.Divergence.Normal},  we have
\begin{equation}
\label{e:exp.Bd.int}
    \inf_{\sA} \sup_{(M, f)} \bE|\h{\eta}(H_n) - \eta(M, f)| \geq \frac{\epsilon}{2} [1 - \sqrt{2 n \epsilon^2}]
\end{equation}
for every $\epsilon \in (0, 1/\sqrt{2n}].$ By substituting $\epsilon = 1/\sqrt{8n},$ the desired relation in \eqref{e:exp.Bd} follows. 


We now prove the variance lower bound, i.e., the case where $\eta(M, f) = \V(M, f)$.
 We again consider two variants $(M_0, f^{-1})$ and $(M_0, f^{+1})$ of the MCP $(M_0, f),$ where $f^{\nu}(s),$ $\nu \in \{-1, +1\},$ is as in \eqref{eq:disc-cost-rv}, but with $\mu = 0$ and $\sigma^2$ replaced by $\sigma_{\nu}^2 := 1 + (1 + \nu) \epsilon$ for some $\epsilon > 0.$ We then have $\V(M_0, f^{+1}) - \V(M_0, f^{-0}) = 2\epsilon.$
	
	Arguing as in \eqref{e:high.prob.lower.Bd}--\eqref{e:KL.Divergence.ind.samples} we then have
	\begin{align}
		\inf_{\sA} \sup_{M, f} &
		\Pr\{|\h{\eta}(H_n, f)  - \eta(M, f)| \geq  \epsilon\} \\
		\geq {} & \frac{1}{2} \left[1 - \sqrt{n D_{\KL}(P_{+1}, P_{-1})}\right] \\
		= {} & \frac{1}{2} \left[1 - \sqrt{n (2\epsilon - \ln (1 + 2\epsilon) + 2\epsilon^2)}\right] \label{e:KL.Divergence.Normal.Variance}\\
		\geq {} & \frac{1}{2} \left[1 - \sqrt{6 n\epsilon^2}\right] \label{e:ln.bd},
	\end{align}
	where \eqref{e:KL.Divergence.Normal.Variance} follows from KL-divergence formula for Gaussian random variables, while \eqref{e:ln.bd} holds 
	since $2\epsilon - \ln(1 + 2\epsilon) \leq 4 \epsilon^2$ for any $\epsilon > 0.$ The rest of the proof follows as in \eqref{e:KL.Divergence.Normal}--\eqref{e:exp.Bd.int}, mutatis mutandis.
\end{proof}

\subsection{Proof of Theorem \ref{thm:ub-cvar-expec}} 
\label{sec:proof-ub-cvar-expec}

\begin{proof}
Let $D_T= \sum_{t = 0}^{T-1} \gamma^t f(s_t)$ denote the infinite-horizon discounted cost random variable for the Markov chain with truncated horizon $T$. Recall that $D= \sum_{t = 0}^\infty \gamma^t f(s_t)$ is the infinite horizon random variable. Let $F$ and $F_T$ denote the CDF of $D$ and $D_T$, respectively.
Then, using the fact that $|f(\cdot)|\le K$, we obtain
\begin{align}
	-\frac{\gamma^T K}{1-\gamma} \le D - D_T \le \frac{\gamma^T K}{1-\gamma}. \label{eq:b1}
\end{align}
It is well-known that CVaR is a coherent risk measure \cite{rockafellar2000optimization}, in particular, satisfying the following two properties for any two random variables $X$ and $Y$: 
(i) $\cvar(X+a)=\cvar(X) + a$ for every $a\in \bR$; and (ii) $\cvar(X) \le \cvar(Y)$ if $X\le Y$ almost surely. 

Using these properties in conjunction with \eqref{eq:b1}, we obtain
\begin{align}
	\cvar(D)-\frac{\gamma^T K}{1-\gamma} \le \cvar(D_T) \le \cvar(D) + \frac{\gamma^T K}{1-\gamma}. \label{eq:b2}
		\end{align}
Using the bound above, we have
\begin{align*}
&\E\left| \hat c_N - \cvar(D)\right| \\
&\le 
    \E\left| \hat c_N - \cvar(D_T)\right| + \E\left|\cvar(D_T) -  \cvar(D)\right|\numberthis\label{eq:a11}\\
&    \le \frac{32 (1-\gamma^T)^2 K^2}{(1-\alpha)(1-\gamma)^2 \sqrt{m}} + \frac{\gamma^T K}{1-\gamma},
\end{align*}
where the final inequality uses \eqref{eq:b2} to bound the second term in \eqref{eq:a11}, while the first term there is bounded using a special case of \cite[Corollary 20]{la2022wasserstein}, which is stated below for the sake of completeness.
\begin{lemma}\label{lemma:cvar-conc}
Let $X_1,\ldots,X_m$ be drawn i.i.d. from the distribution of a random variable $X$, satisfying $|X|\le B$ a.s. Let $\hat{c}_{m, \alpha}$ be formed using \eqref{eq:var-cvar-estimate}. Then,
\begin{align}
    \E\left| \hat{c}_{m, \alpha} - \cvar(X)\right| \le \frac{32B^2}{(1-\alpha) \sqrt{m}}.
\end{align}
\end{lemma}
The application of the result above to bound the first term in \eqref{eq:a11} is valid since the truncated trajectories are independent, and $|D_T| \le \frac{(1-\gamma^T) K}{1-\gamma}$ a.s. 
\end{proof}

\subsection{Proof of Theorem \ref{thm:ub-cvar}} 
\label{sec:proof-sec:proof-ub-cvar}
\begin{proof}
	The initial passage in the proof of Theorem \ref{thm:ub-cvar-expec} leading up to \eqref{eq:b2} holds. 
Using the inequalities in \eqref{eq:b2}, we arrive at the main claim as follows:
\begin{align*}
&\prob{| \hat c_N - \cvar(D)| > \epsilon} \\
&= \prob{| \hat c_N - \cvar(D_T) + \cvar(D_T) -  \cvar(D)| > \epsilon} \\
& \le \prob{| \hat c_N - \cvar(D_T)| > \epsilon - \frac{\gamma^T K}{1-\gamma}}\\
& \le 6 \exp\left(- \frac{m (1-\alpha)(1-\gamma)^2}{44 (1-\gamma^T)^2K^2} \left(\epsilon-\frac{\gamma^T K}{1-\gamma}\right)^2\right),
\end{align*}
where the penultimate inequality used \eqref{eq:b2}, while the final inequality follows from the CVaR concentration bound in Theorem 3.1 of \cite{wang2010deviation}. 
\end{proof}


\subsection{Proof of Theorem \ref{thm:ub-var}}
\label{sec:proof-sec:proof-ub-var}
\begin{proof}
As in the CVaR case, we have 
(i) $\var(X+a)=\var(X) + a$ for every $a\in \bR$; and (ii) $\var(X) \le \var(Y)$ if $X\le Y$ almost surely. 
Using these facts, we obtain 
\begin{align}
	\var(D)-\frac{\gamma^T K}{1-\gamma} \le \var(D_T) \le \var(D) + \frac{\gamma^T K}{1-\gamma}. \label{eq:b24}
		\end{align}
For the sake of completeness, we state a concentration bound for VaR in the i.i.d. case from \cite{prashanth2020concentration} below.
\begin{lemma}\label{lemma:var-conc}
Let $X_1,\ldots,X_m$ be drawn i.i.d. from the distribution of a random variable $X$, satisfying (A1). Let $\hat{v}_{m, \alpha}$ be formed using \eqref{eq:var-cvar-estimate}. Then,
For any $\epsilon > 0,$ we have
	\begin{align}
	\prob{\vert \hat{v}_{m, \alpha} - \var(X) \vert \geq \epsilon} \leq 2 \exp \left(  -2n\ell^2\min\left(\epsilon^2,\zeta^2\right)   \right), \label{eq:var-conc-bd1}
	\end{align}
	where $\ell,\zeta$ are specified in (A1).
\end{lemma}
The claim follows follows by using the inequality in \eqref{eq:b24} in conjunction with Lemma \ref{lemma:var-conc}.
\end{proof}

\subsection{Proof of Theorem \ref{thm:expec-bd-generalrisk}}
\label{sec:appendix-ub-lipschitz-proof}
\begin{proof}
	The initial passage in the proof of Theorem \ref{thm:ub-cvar-expec} leading up to \eqref{eq:b2} holds for a Lipschitz risk measure $\eta(\cdot)$ satisfying properties (i) and (ii) from the theorem statement. Thus,
	\begin{align}
		\eta(D)-\frac{\gamma^T K}{1-\gamma} \le \eta(D_T)
		\le \eta(D) + \frac{\gamma^T K}{1-\gamma}. \label{eq:b3}
	\end{align}
	Using the bound above, we have
	\begin{align*}
		\E\left| \hat\eta_N - \eta(D)\right| 
		&\le 
		\E\left| \hat\eta_N - \eta(D_T)\right| + \E\left|\eta(D_T) -  \eta(D)\right|\numberthis\label{eq:a112}\\
		&    \le \frac{32 L (1-\gamma^T)^2 K^2}{(1-\gamma)^2 \sqrt{m}} + \frac{\gamma^T K}{1-\gamma},
	\end{align*}
	where the final inequality uses \eqref{eq:b3} to bound the second term in \eqref{eq:a112}, while the first term there is bounded using a special case of the result from \cite[Theorem 19]{la2022wasserstein}, which is given below.
	\begin{lemma}
		Let $X_1,\ldots,X_m$ be drawn i.i.d. from the distribution of a random variable $X$, satisfying $|X|\le B$ a.s. Suppose $\eta(\cdot)$ is a Lipschitz risk measure with constant $L$. Let $\hat\eta_m=\eta(F_m)$, where $F_m$ is the EDF.  Then,
		\begin{align}
			\E\left| \hat\eta_m - \eta(X)\right| \le \frac{32L B^2}{ \sqrt{m}}.
		\end{align}
	\end{lemma}	
	The application of the result above to bound the first term in \eqref{eq:a112} is valid for reasons listed at the end of the proof of Theorem \ref{thm:ub-cvar-expec}.
\end{proof}

\subsection{Proof of Theorem \ref{thm:conc-bd-generalrisk}}
\label{sec:appendix-ub-lipschitz-proof-conc}
\begin{proof}
	The initial passage in the proof of Theorem \ref{thm:expec-bd-generalrisk} leading up to \eqref{eq:b3} holds here. Using the inequalities in \eqref{eq:b3}, we derive the main concentration result as follows:
	\begin{align*}
		\prob{| \hat\eta_N - \eta(D)| > \epsilon} 
		&= \prob{| \hat\eta_N - \eta(D_T) + \eta(D_T) -  \eta(D)| > \epsilon} \\
		& \le \prob{| \hat\eta_N - \eta(D_T)| > \epsilon - \frac{\gamma^T K}{1-\gamma}}\numberthis \label{eq:st12}\\
		& \le 2 \exp\left(- \frac{m(1-\gamma)^2}{256 (1-\gamma^{T})^2 K^2\myexp} \left[\frac{1}{L} \left[\epsilon-\frac{\gamma^T K}{1-\gamma}\right]-\frac{512K}{(1-\gamma)\sqrt{m}}\right]^2\right),
	\end{align*}
	where the penultimate inequality used \eqref{eq:b3}, while the final inequality follows by applying Theorem 27 in \cite{la2022wasserstein} after observing that $|D_T| \le \frac{(1-\gamma^T) K}{1-\gamma}$.
	For the sake of completeness, we state the aforementioned result from  \cite{la2022wasserstein} below.
	\begin{lemma}\label{lemma:var-conc}
		Let $X_1,\ldots,X_m$ be drawn i.i.d. from the distribution of a sub-Gaussian random variable $X$ with parameter $\sigma$. Suppose $\eta(\cdot)$ is a Lipschitz risk measure with constant $L$. Let $\hat\eta_m=\eta(F_m)$, where $F_m$ is the EDF. Then, for any $  \frac{512\sigma}{\sqrt{n}}<\left(\frac{\epsilon}{L}\right) < \frac{512\sigma}{\sqrt{n}}+16\sigma\sqrt{\myexp}$, we have 
		\begin{align}
			\prob{\vert \eta_{m} - \eta(X) \vert \geq \epsilon} \leq \exp\left(- \frac{m}{256\sigma^2\myexp} \left(\left(\frac{\epsilon}{L}\right)-\frac{512\sigma}{\sqrt{m}}\right)^2\right). \label{eq:var-conc-bd1}
		\end{align}
	\end{lemma}
The application of the result above to bound \eqref{eq:st12} is valid since $D_T$ is a bounded random variable, implying sub-Gaussianity.
\end{proof}

\subsection{Proof of Theorem \ref{thm:ub-variance-conc}}
\label{sec:appendix-ub-variance-proof-conc}
\begin{proof}
	Follows in a similar manner as the proof of Theorem \ref{thm:ub-cvar} after observing that \\[0.5ex]
	\centerline{$|\V(D)-\V(D_T)| \le \frac{4\gamma^T K^2}{(1-\gamma)^2}.$}
\end{proof}

\section{Conclusions}
\label{sec:conclusions}
We studied estimation of risk measures such as \vart, \cvart, mean and variance in an infinite-horizon discounted MCP.  We provided minimax lower bounds for estimating an $\epsilon$-accurate solution, both in a high-probability and expected sense.  We have also proposed a truncation-based estimator for the aforementioned risk measures, and obtained an upper bound on its sample complexity. 

\bibliography{references}

\begin{thebibliography}{21}
\providecommand{\natexlab}[1]{#1}
\providecommand{\url}[1]{\texttt{#1}}
\expandafter\ifx\csname urlstyle\endcsname\relax
  \providecommand{\doi}[1]{doi: #1}\else
  \providecommand{\doi}{doi: \begingroup \urlstyle{rm}\Url}\fi

\bibitem[Acerbi(2002)]{acerbi2002spectral}
C.~Acerbi.
\newblock Spectral measures of risk: A coherent representation of subjective
  risk aversion.
\newblock \emph{Journal of Banking \& Finance}, 26\penalty0 (7):\penalty0
  1505--1518, 2002.

\bibitem[Artzner et~al.(1999)Artzner, Delbaen, Eber, and
  Heath]{artzner1999coherent}
P.~Artzner, F.~Delbaen, J.~Eber, and D.~Heath.
\newblock {Coherent measures of risk}.
\newblock \emph{Mathematical Finance}, 9\penalty0 (3):\penalty0 203--228, 1999.

\bibitem[Ben-Tal and Teboulle(1986)]{bental1986oce}
A.~Ben-Tal and M.~Teboulle.
\newblock Expected utility, penalty functions, and duality in stochastic
  nonlinear programming.
\newblock \emph{Management Science}, 32\penalty0 (11):\penalty0 1445–1466,
  Nov. 1986.
\newblock ISSN 0025-1909.

\bibitem[Ben-Tal and Teboulle(2007)]{bental2007oce}
A.~Ben-Tal and M.~Teboulle.
\newblock An old-new concept of convex risk measures: The optimized certainty
  equivalent.
\newblock \emph{Mathematical Finance}, 17:\penalty0 449--476, 02 2007.

\bibitem[Bertsekas and Tsitsiklis(1996)]{BertsekasT96}
D.~P. Bertsekas and J.~N. Tsitsiklis.
\newblock \emph{Neuro-Dynamic Programming}.
\newblock Athena Scientific, 1996.

\bibitem[{Bhat} and Prashanth(2019)]{bhat2019concentration}
S.~P. {Bhat} and L.~A. Prashanth.
\newblock {Concentration of risk measures: A Wasserstein distance approach}.
\newblock In \emph{Advances in Neural Information Processing Systems}, pages
  11739--11748, 2019.

\bibitem[Brown(2007)]{brown2007large}
D.~B. Brown.
\newblock Large deviations bounds for estimating conditional value-at-risk.
\newblock \emph{Operations Research Letters}, 35\penalty0 (6):\penalty0
  722--730, 2007.

\bibitem[Duchi(2024)]{Duchi2023statistics}
J.~Duchi.
\newblock Lectures notes on statistics and information theory.
\newblock \url{https://web.stanford.edu/class/stats311/lecture-notes.pdf},
  2024.
\newblock Accessed: 2024-02-02.

\bibitem[F{\"o}llmer and Schied(2002)]{follmer2002convex}
H.~F{\"o}llmer and A.~Schied.
\newblock Convex measures of risk and trading constraints.
\newblock \emph{Finance and stochastics}, 6\penalty0 (4):\penalty0 429--447,
  2002.

\bibitem[Hu and Zhang(2018)]{hu2018utility}
Z.~Hu and D.~Zhang.
\newblock Utility-based shortfall risk: Efficient computations via monte carlo.
\newblock \emph{Naval Research Logistics (NRL)}, 65\penalty0 (5):\penalty0
  378--392, 2018.

\bibitem[Kolla et~al.(2019{\natexlab{a}})Kolla, Prashanth, Bhat, and
  Jagannathan]{Kolla}
R.~K. Kolla, L.~Prashanth, S.~P. Bhat, and K.~Jagannathan.
\newblock Concentration bounds for empirical conditional value-at-risk: The
  unbounded case.
\newblock \emph{Operations Research Letters}, 47\penalty0 (1):\penalty0 16 --
  20, 2019{\natexlab{a}}.

\bibitem[Kolla et~al.(2019{\natexlab{b}})Kolla, Prashanth, Bhat, and
  Jagannathan]{kolla2019concentration}
R.~K. Kolla, L.~A. Prashanth, S.~P. Bhat, and K.~P. Jagannathan.
\newblock Concentration bounds for empirical conditional value-at-risk: The
  unbounded case.
\newblock \emph{Operations Research Letters}, 47\penalty0 (1):\penalty0 16--20,
  2019{\natexlab{b}}.

\bibitem[Metelli et~al.(2023)Metelli, Mutti, and Restelli]{metelli2023tale}
A.~M. Metelli, M.~Mutti, and M.~Restelli.
\newblock A tale of sampling and estimation in discounted reinforcement
  learning.
\newblock In \emph{International Conference on Artificial Intelligence and
  Statistics}, pages 4575--4601. PMLR, 2023.

\bibitem[Prashanth and Bhat(2022)]{la2022wasserstein}
L.~Prashanth and S.~P. Bhat.
\newblock {A Wasserstein Distance Approach for Concentration of Empirical Risk
  Estimates}.
\newblock \emph{Journal of Machine Learning Research}, 23\penalty0
  (238):\penalty0 1--61, 2022.

\bibitem[Prashanth and Fu(2022)]{prashanth2022risk}
L.~A. Prashanth and M.~C. Fu.
\newblock Risk-sensitive reinforcement learning via policy gradient search.
\newblock \emph{Foundations and Trends{\textregistered} in Machine Learning},
  15\penalty0 (5):\penalty0 537--693, 2022.

\bibitem[Prashanth et~al.(2020)Prashanth, Jagannathan, and
  Kolla]{prashanth2020concentration}
L.~A. Prashanth, K.~Jagannathan, and R.~K. Kolla.
\newblock {Concentration bounds for CVaR estimation: The cases of light-tailed
  and heavy-tailed distributions}.
\newblock In \emph{{International Conference on Machine Learning}}, volume 119,
  pages 5577--5586, 2020.

\bibitem[Rockafellar and Uryasev(2000)]{rockafellar2000optimization}
R.~T. Rockafellar and S.~Uryasev.
\newblock {Optimization of conditional value-at-risk}.
\newblock \emph{Journal of Risk}, 2:\penalty0 21--42, 2000.

\bibitem[Serfling(2009)]{serfling2009approximation}
R.~J. Serfling.
\newblock \emph{Approximation theorems of mathematical statistics}, volume 162.
\newblock John Wiley \& Sons, 2009.

\bibitem[Sutton and Barto(2018)]{sutton2018reinforcement}
R.~S. Sutton and A.~G. Barto.
\newblock \emph{Reinforcement Learning: An Introduction}.
\newblock MIT Press, Cambridge, MA, 2nd ed. edition, 2018.

\bibitem[Thomas and Learned-Miller(2019)]{thomas2019concentration}
P.~Thomas and E.~Learned-Miller.
\newblock Concentration inequalities for conditional value at risk.
\newblock In \emph{International Conference on Machine Learning}, pages
  6225--6233, 2019.

\bibitem[Wang and Gao(2010)]{wang2010deviation}
Y.~Wang and F.~Gao.
\newblock Deviation inequalities for an estimator of the conditional
  value-at-risk.
\newblock \emph{Operations Research Letters}, 38\penalty0 (3):\penalty0
  236--239, 2010.

\end{thebibliography}

\end{document}